\documentclass[twocolumn, aps, superscriptaddress, groupedaddress,nofootinbib]{revtex4}  

\usepackage{times}
\usepackage{helvet}
\usepackage{courier}
\usepackage{graphicx}
\usepackage{graphicx}

\usepackage{times}
\usepackage{helvet}
\usepackage{courier}
\usepackage{url}
\usepackage{arydshln}
\usepackage{amsmath,amsthm,amsfonts,amssymb}
\usepackage{fontenc}
\usepackage[all]{xy}
\usepackage{graphicx,subfigure,multirow}
\usepackage{tikz}
\usetikzlibrary{shapes, arrows, shadows}
\usepackage{dsfont}
\usepackage{algorithm}
\usepackage[noend]{algpseudocode}

\usepackage[utf8]{inputenc}
\usepackage[english]{babel}
\usepackage{booktabs}
\usepackage{cancel}
\usepackage[shortlabels]{enumitem}
\usepackage{txfonts}

\usepackage{amsmath,amsthm,amsfonts,amssymb}   
\usepackage{fontenc}
\usepackage[all]{xy}
\usepackage{graphicx,subfigure,multirow}
\usepackage{tikz}
\usetikzlibrary{shapes, arrows, shadows} 
\usepackage{graphicx}				
\usepackage{amssymb}
\usepackage[margin=2.54cm, a4paper]{geometry}
\usepackage{amsmath,amsthm,amsfonts,amssymb}
\usepackage{fontenc}
\usepackage[all]{xy}
\usepackage{graphicx,subfigure,multirow}
\usepackage{tikz}
\usetikzlibrary{shapes, arrows, shadows} 
\usepackage{dsfont}
\usepackage{algorithm}
\usepackage[noend]{algpseudocode}

\usepackage[utf8]{inputenc}
\usepackage{booktabs}
\usepackage{cancel}
\usepackage[shortlabels]{enumitem}
\usepackage{txfonts}

\newcommand{\CI}{\mathrel{\perp\mspace{-10mu}\perp}}

\usepackage{subfiles}

\graphicspath{{images/}{../images/}}

\usetikzlibrary{bayesnet} 

\newtheorem{theorem}{Theorem}

\begin{document}

\title{Integrating overlapping datasets using bivariate causal discovery}

\author{Anish Dhir}
\affiliation{Babylon Health, London, United Kingdom}
\author{Ciar{\'a}n M. Lee}
\affiliation{Babylon Health, London, United Kingdom}
\affiliation{University College London, United Kingdom}

\begin{abstract}
 Causal knowledge is vital for effective reasoning in science, as causal relations, unlike correlations, allow one to reason about the outcomes of interventions. Algorithms that can discover causal relations from observational data are based on the assumption that all variables have been jointly measured in a single dataset. In many cases this assumption fails. Previous approaches to overcoming this shortcoming devised algorithms that returned all joint causal structures consistent with the conditional independence information contained in each individual dataset. But, as conditional independence tests only determine causal structure up to Markov equivalence, the number of consistent joint structures returned by these approaches can be quite large. The last decade has seen the development of elegant algorithms for discovering causal relations beyond conditional independence, which can distinguish among Markov equivalent structures. In this work we adapt and extend these so-called bivariate causal discovery algorithms to the problem of learning consistent causal structures from multiple datasets with overlapping variables belonging to the same generating process, providing a sound and complete algorithm that outperforms previous approaches on synthetic and real data.
 

\end{abstract}

\maketitle

\section{Introduction} \label{Section: introduction}

Causal knowledge is fundamental to many domains of science and medicine. This is due to the fact that causal relations, unlike correlations, allow one to reason counterfactually and to analyse the consequences of interventions \cite{P,richens2019,perov2019}. While powerful approaches to discovering causal relations between multiple variables in the absence of randomised controlled trials have been developed \cite{Noise,shimizu2006linear,janzing2012information,mitrovic2018causal,Confounders,goudet2017learning,zhang2009identifiability,fonollosa2016conditional,lee2017causal,gasse2012experimental,tsamardinos2006max,P,Sprite,kalainathan2018sam}, many of these require all variables to be jointly measured in a single dataset. In many domains this assumption does not hold, due to ethical concerns, or technological constraints. For instance, in certain countries medical variables could be censored differently, meaning we only have access to joint measurements of certain variables; distinct medical sensors may measure different but overlapping aspects of a particular disease or physiological function, for example fMRI machines are unable to obtain measurements for every region of the brain simultaneously \cite{tillman2011learning}.
In these examples, we are provided with multiple datasets, each recording a potentially different, but overlapping, set of variables. Such overlapping datasets occur frequently in medicine---clinical studies require strict ethical approval which restricts the ability to measure variables not required by the study, resulting in many studies which measure overlapping, but not coinciding, variables. Can these datasets be combined in such a way that causal relations between non-overlapping variables---which have never been jointly measured---be discovered? 

This problem was first studied in \cite{danks2009integrating}, with the state of the art achieved by the \textit{integration of overlapping datasets} (IOD) algorithm of \cite{tillman2011learning}. The authors employed conditional independence tests to learn the Markov equivalence class of each individual dataset, using this to determine the equivalence classes of consistent joint structures among all variables in the union of datasets. 
A `consistent' joint structure is one whose conditional independences do not contradict those already learned from each individual dataset.

Approaches based on conditional independence tests are limited as they can only determine causal structures up to a Markov equivalence class.
Due to this, conditional independence approaches result in too many consistent answers, especially when the number of overlapping variables is small. 
They also fail to distinguish multiple causal structures between small numbers of variables, such as those depicted in Fig.~\ref{figure:causal sufficiency}---as they all belong to the same Markov equivalence class.
As many medical studies individually measure relatively small numbers of variables, this is a major roadblock to extracting useful causal information. For instance, \cite{savastano2017low} provide clinical studies that support the role of obesity in the development of low vitamin D, while \cite{despres2006abdominal} provide separate studies which associate obesity with an increased risk of heart failure. To extract useful information about the causal relationship between heart failure and vitamin D deficiency from these two overlapping datasets, methods beyond conditional independence tests are required.

Fortunately, the last decade has seen the development of elegant algorithms for discovering causal relations which go beyond conditional independence and can distinguish different members of the same Markov equivalence class---assuming all variables have been jointly measured \cite{Noise,shimizu2006linear,janzing2012information,mitrovic2018causal,Confounders,goudet2017learning,zhang2009identifiability,fonollosa2016conditional}. These algorithms, termed bivariate causal discovery, employ tools from machine learning and assumptions about what it means for one variable to cause another, allowing for a more fine-grained approach to discovering causal relations. 

This paper expands and adapts bivariate causal discovery algorithms to the problem of learning consistent causal structures from overlapping datasets. 
Our main contributions are as follows: 
\begin{enumerate}
    \item A sound and complete algorithm for learning causal structure from overlapping datasets that leads to fewer consistent structures when compared against previous approaches.
    \item A robust comparison between our approach and the IOD algorithm on a range of synthetic and real world data. These cover the regimes of low overlap and low number of variables, low overlap and high number of variables, and high overlap and high number of variables.
    We also include a performance comparison of the algorithms as a function of the number of overlapping variables.
\end{enumerate}

\section{Related work} \label{Section: related work}

\textbf{Discovering causal structure from a single dataset:}
Methods for discovering causal structure from a single i.i.d. dataset largely fall into two categories. The first is \emph{global} causal discovery, which aims to learn a partially undirected version of the underlying DAG. There are two distinct approaches to this: constraint and score based. The constraint based approach uses conditional independence tests to determine which variables should share an edge in the causal structure. Examples include the PC \cite{Sprite}, IC \cite{P}, and FCI \cite{spirtes1995causal} algorithms. 
The score based approach utilizes a scoring function, such as Minimum Description Length, to evaluate each network with respect to some training data, and searches for the optimal network according to this function \cite{friedman1997bayesian}. 

\begin{figure}[t] 
\centering
\begin{subfigure}
\centering
(a)\includegraphics[scale=.24]{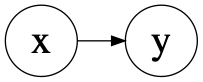}
\end{subfigure}
\begin{subfigure}
\centering
(b) \includegraphics[scale=0.24]{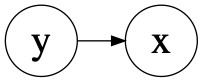}
\end{subfigure}
\centering
\begin{subfigure}
(c) \includegraphics[scale=0.24]{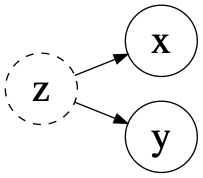}
\end{subfigure}
\centering
\begin{subfigure}
\ \ \  \ \ (d) \includegraphics[scale=0.24]{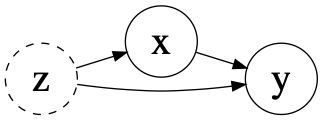}
\end{subfigure}
\centering
\begin{subfigure}
(e) \includegraphics[scale=0.24]{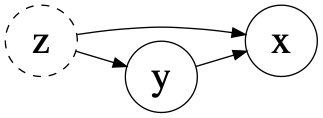}
\end{subfigure}
\caption{All causal structures between two correlated variables, solid nodes observed \& dashed latent.}\label{figure:causal sufficiency}
\end{figure} 

\textbf{Bivariate causal discovery on a single dataset:} The main limitation of global discovery algorithms is that they cannot always orient edges between dependent variables. That is, they can only learn causal structure up to Markov equivalence class. In particular, they cannot distinguish any of the structures in Fig.~\ref{figure:causal sufficiency}. The second category of causal discovery algorithm, termed \emph{bivariate causal discovery} (BCD), aims to overcome this by specifying some assumptions which---if satisfied---make the intuitive asymmetry between cause and effect manifest at an observational level. Examples include the Linear Non-Gaussian Additive Model (LiNGAM) \cite{shimizu2006linear}, Additive Noise Model (ANM)  \cite{Noise}, information geometric causal discovery algorithm \cite{daniusis2012inferring}, and the kernel conditional deviance causal discovery (KCDC) algorithm \cite{mitrovic2018causal}. 

\textbf{Discovering causal structure from multiple overlapping datasets:} The first algorithm for learning causal structure from overlapping datasets was \emph{integration of overlapping networks} (ION) \cite{danks2009integrating}.
This was extended and improved by the \textit{integration of overlapping datasets} (IOD) algorithm \cite{tillman2011learning}.
The IOD algorithm takes in multiple datasets, where the variables of each dataset have a non-empty intersection with the union of the variables from the remaining datasets.  
It is assumed that all the datasets belong to the same data generating process. This implies that the datasets do not entail fundamentally contradictory information.
The aim of the algorithm is then to infer consistent graphical structures over all variables from the overlapping variable datasets. The fundamental insight behind IOD is that every graph constructed for each dataset is a marginalised version of the full graph between all variables. Thus, every graph that marginalises to the graphs constructed using each dataset is a candidate for the graphical structure of the true data generating process.
Instead of marginalising every candidate graph and then checking consistency, IOD uses graphical criteria on the full candidate graph to check for consistency, as will be detailed in Section~\ref{Section: methods}.
This results in candidate graphs over all variables, that encode the same conditional independence information as the individual datasets. Various extensions and modifications have also been developed \cite{tsamardinos2012towards,triantafillou2015constraint,sajja2015bayesian,claassen2010causal,janzing2018merging}. 





\section{Motivating Example} \label{Section:motivating_example}

\begin{figure}[t]   
\begin{subfigure}
\centering
(a)\includegraphics[scale=.24]{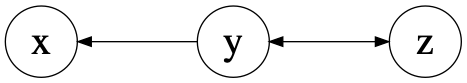} \\
\end{subfigure}
\begin{subfigure}
\centering
 (b) \includegraphics[scale=0.24]{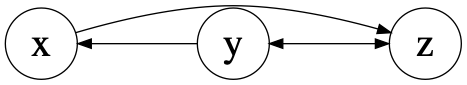} \\
\end{subfigure}
\centering
\begin{subfigure}
(c) \includegraphics[scale=0.24]{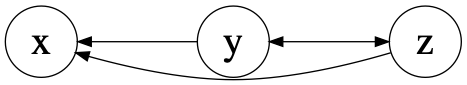}
\end{subfigure}
\caption{(a) Consistent joint structure. (b) \& (c) Joint structures ruled out by causal sufficiency.}\label{figure:motivating example}
\end{figure} 

 The following example illustrates the power of 
bivariate casual discovery to learn consistent causal structures from overlapping datasets. The basic mechanics of our approach, and how it leverages bivariate causal discovery, anticipates our main algorithm, which is described in Section~\ref{subSection: bivariate causal discovery}. 

Consider two datasets, $\{X,Y\}$ and $\{Y,Z\}$. Our aim is to learn all consistent joint causal structures involving these three variables. 
We define a \emph{causally sufficient bivariate causal discovery algorithm} as one that can take in samples from two variables  $X,Y$ and output the correct causal structure from all the possible ones in Fig.~\ref{figure:causal sufficiency}---or is able to find all the structures from Fig.~\ref{figure:causal sufficiency} present in the dataset.
For now, we treat this algorithm as an oracle and discuss the justification of causal sufficiency at the end of this section. 

Suppose that using this oracle we find $X\longleftarrow Y$, that is $Y$ causes $X$, and $Y \longleftrightarrow Z$, that is $Y$ and $Z$ share a latent common cause, as illustrated in Fig.~\ref{figure:motivating example}(a).
As $X$ and $Z$ have never been jointly measured and could each have been measured under different conditions, we cannot use causal discovery on them.
However, we can posit a causal structure between them and check if this structure is consistent with the causal information extracted from the individual datasets. 
For instance, positing that $X$ causes $Z$, as depicted in Fig.~\ref{figure:motivating example}(b), leads to a direct cause from $Y$ to $Z$---mediated by $X$---as well as the original common cause between $Y$ and $Z$.
That is, the structure is as depicted in Fig.~\ref{figure:causal sufficiency}(d) for appropriate node labels. But this contradicts the original marginal structure, which detected only a common cause between $Y$ and $Z$, as depicted in Fig.~\ref{figure:causal sufficiency}(c).
As the algorithm used to learn this structure was promised to be causally sufficient, and hence can distinguish all structures in Fig.~\ref{figure:causal sufficiency}, we conclude $X$ cannot be a cause of $Z$. We can similarly conclude that $Z$ cannot be a cause of $X$. Hence, neither joint causal structures in Fig.~\ref{figure:motivating example}(b)\&(c) are consistent with the structures learned from each dataset. 

If we had employed previous approaches to learning consistent joint causal structures, such as ION \cite{danks2009integrating} and IOD \cite{tillman2011learning}, discussed in Section~\ref{Section: related work}, we would have been unable to rule out the structures in Fig.~\ref{figure:motivating example}(b) \& (c). Indeed, as ION and IOD employ conditional independence tests to discover underlying causal structure, they would not detect any contradiction between these joint structures and the marginal structures of each dataset. That is, due to the fact that all causal structure from Fig.\ref{figure:causal sufficiency} belong to the same Markov equivalence class, conditional independence tests cannot distinguish between directed causes and simultaneously direct-and-common causes. 

\subsection{Causal Sufficiency} 
This simple example suggests that exploiting causally sufficient bivariate causal discovery algorithms allows us to output a smaller solution set of joint causal structures consistent with individual datasets---thus getting closer to the true structure. While there are many algorithms which can only distinguish between the causal structures in Fig. \ref{figure:causal sufficiency} (a) \& (b) \cite{Noise,daniusis2012inferring,mitrovic2018causal}, there are a few bivariate causal discovery algorithms which are causally sufficient under certain assumptions, such as \cite{janzing2018detecting,hoyer2008estimation}.
For example, \cite{hoyer2008estimation} requires a non-Gaussian noise term and linear relationships between all the variables. \cite{janzing2018detecting} relies on a concentration of measure assumption and a scalar confounder.
Moreover, \cite{kalainathan2018sam,goudet2017learning} have tested the robustness of their algorithms to unobserved confounding, showing they can in some instances detect the presence of latent common causes.
These can be used to identify when algorithms that can only distinguish Fig. \ref{figure:causal sufficiency} (a) \& (b) can safely be applied. 
Causal sufficiency can thus be achieved under certain assumptions with current causal discovery algorithms. Additionally, expert knowledge and intervention-based studies, such as \cite{sachs2005causal}, can provide causally sufficient information. For instance, all we needed for Fig.~\ref{figure:motivating example}(b) to be ruled out was the knowledge that $X$ could not be a mediator between $Y$ and $Z$. In some domains, such as medicine, experts may provide such information. This knowledge could be used in conjunction with non-causally sufficient bivariate algorithms to achieve the same conclusions as above. 
Additionally, as causal discovery may rely on multiple assumptions and interventions decouple common cause influence, intervention-based studies can detect the presence of latent confounding and provide causally sufficient information.





\vspace{-.2cm}
\section{Methods} \label{Section: methods}

We now provide a description of the IOD algorithm. 
We then describe our approach to extending and exploiting bivariate causal discovery. We assume familiarity with graphical terminology, including active paths, m-separation, MAGs, PAGs, and related concepts. An overview of the required background terminology is provided in the Appendix.

\begin{table*}[t]
    \centering
    \begin{tabular}{ccccc}
        \textbf{Criterion 1, Fig.~\ref{figure:causal sufficiency}(a):} & $\quad
\text{\textbf{msep}}\left(X,Y\right)_{\overline{Y}}$ & $\quad \cancel{\text{\textbf{msep}}}\left(X,Y\right)_{\underline{Y}}$ & $\quad  \cancel{\text{\textbf{msep}}}\left(X, Y\right)_{\overline{X}}$ &  $\quad \text{\textbf{msep}}\left(X, Y\right)_{\underline{X}}$   \\
\textbf{Criterion 2, Fig.~\ref{figure:causal sufficiency}(b):} & $\quad
\cancel{\text{\textbf{msep}}}\left(X, Y\right)_{\overline{Y}}$ & $\quad \text{\textbf{msep}}\left(X Y\right)_{\underline{Y}}$ & $ \quad  \text{\textbf{msep}}\left(X, Y\right)_{\overline{X}}$ & $ \quad \cancel{\text{\textbf{msep}}}\left(X, Y\right)_{\underline{X}}$ \\
\textbf{Criterion 3, Fig.~\ref{figure:causal sufficiency}(c):} & $\quad
\text{\textbf{msep}}\left(X, Y\right)_{\overline{Y}}$ & $\quad \cancel{\text{\textbf{msep}}}\left(X,Y\right)_{\underline{Y}}$ & $\quad  \text{\textbf{msep}}\left(X, Y\right)_{\overline{X}}$ & $\quad \cancel{\text{\textbf{msep}}}\left(X, Y\right)_{\underline{X}}$ \\
\textbf{Criterion 4, Fig.~\ref{figure:causal sufficiency}(d):} & $\quad \text{\textbf{msep}}\left(X, Y\right)_{\overline{Y}}$ & $\quad \cancel{\text{\textbf{msep}}}\left(X, Y\right)_{\underline{Y}}$ & $ \quad  \cancel{\text{\textbf{msep}}}\left(X, Y\right)_{\overline{X}}$ & $\quad \cancel{\text{\textbf{msep}}}\left(X, Y\right)_{\underline{X}}$ \\
\textbf{Criterion 5, Fig.~\ref{figure:causal sufficiency}(e):} & $\quad
\cancel{\text{\textbf{msep}}}\left(X,Y\right)_{\overline{Y}}$ & $ \quad \cancel{\text{\textbf{msep}}}\left(X, Y\right)_{\underline{Y}}$ & $\quad  \text{\textbf{msep}}\left(X, Y\right)_{\overline{X}}$ & $\quad \cancel{\text{\textbf{msep}}}\left(X, Y\right)_{\underline{X}}$ \\
    \end{tabular}
    \caption{$(\dots)_{\overline{Z}}$ denotes a condition in a MAG with incoming edges removed from node $Z$, $(\dots)_{\underline{Z}}$ denotes outgoing edges removed, $\text{\textbf{msep}}(X,Y)$ denotes that $X$\&$Y$ are m-separated, and $\cancel{\text{\textbf{msep}}}(X,Y)$ that X\&Y are not m-separated. }
    \label{tab:criteria}
\end{table*}

\subsection{IOD Algorithm} \label{subSection: iod algorithm}
We informally described IOD in Section~\ref{Section: related work}. More formally, IOD can be broken into two parts:

\textbf{Part 1)} 
Starting from fully connected, unoriented graphs $\mathcal{G}_1, \dots, \mathcal{G}_n$ for each variable set $\mathcal{V}_1, \dots, \mathcal{V}_n$, the first few steps of the FCI algorithm \cite{spirtes1995causal} are applied. Edges can be dropped and immoralities oriented using conditional independence tests on each dataset.
These processes are also carried out on a fully connected graph, $\mathcal{G}$, containing all the variables $\mathcal{V} = \bigcup^{n}_{i=1}\mathcal{V}_i$.
During this step, any conditional independence information, along with the conditioning set, are stored in a data structure called \textbf{Sepset}.
If two dependent variables from $\mathcal{V}_i$ are not conditionally independent given any other set of nodes, the pair is added, along with $\mathcal{V}_i$, to a data structure called \textbf{IP}.\footnote{In reality this step accesses a set \textbf{Possep} to obtain the independence information, see \cite{tillman2011learning}.}
To improve the robustness of the independence tests across datasets, the p-values of the tests for overlapping variables are pooled using Fisher's method \cite{fisher1992statistical}.
The output of this step is the graphs, $\mathcal{G}, \mathcal{G}_1, \dots \mathcal{G}_n$ along with the data structures \textbf{Sepset} and \textbf{IP}.

\textbf{Part 2)} 
This part requires graphs $\mathcal{G}, \mathcal{G}_1, \dots \mathcal{G}_n$ along with \textbf{Sepset} and \textbf{IP}.
The global graph $\mathcal{G}$ now consists of a superset of edges and a subset of immoralities compared to the true data generating graph.
This motivates the next step, which considers edges to remove in $\mathcal{G}$ and, within the resulting graph, constructs immoralities with the unoriented edges.
Conditions for the sets \textbf{edges to remove} and \textbf{immoralities to orient} are given in \cite{tillman2011learning}.
As we do not know which combination of removed edges and oriented immoralities is the correct one, this step requires nested iterations over powersets of both sets.
At each iteration, $\mathcal{G}$ is converted to a PAG using the rules in \cite{Zhang:2007:CME:3020488.3020543}, which finds all invariant tails and arrows. The PAG is then converted to a MAG in its equivalence class and it is checked whether: (1) It is indeed a MAG, (2) m-separations in the MAG are consistent with those in \textbf{Sepset}, and (3) there is an inducing path between every pair in \textbf{IP} with respect to $\mathcal{V} \backslash \mathcal{V}_i$.
If a MAG satisfies all conditions, the corresponding PAG marginalises to the dataset PAGs, and is returned as a candidate true graph.

See Section 2 in Appendix and the Synthetic 2 experiment for a discussion of complexity.

\subsection{New approach leveraging bivariate causal discovery} \label{subSection: bivariate causal discovery}   

As we saw in Section~\ref{Section:motivating_example}, causally sufficient bivariate causal discovery appears\footnote{As noted by \cite{zhang2011discussion}, care should be taken when applying bivariate causal discovery to overlapping datasets. The identifiability requirements of such algorithms should hold under marginalisation of the graph. Employing algorithms which do not make parametric assumptions on the model---as in \cite{mitrovic2018causal}---is vital.} to greatly reduce the number of joint causal structures consistent with the causal information extracted from individual datasets. The reason for this is twofold. Firstly, bivariate causal discovery algorithms allow us to determine the specific member of the Markov equivalence class our local datasets belong to. Secondly, the ability, given some assumptions (Sec. \ref{Section:motivating_example}), to distinguish all structures in Fig.~\ref{figure:causal sufficiency} provides us with a richer set of conditions---beyond conditional independence---to check in order to ensure consistency.

The algorithm developed here assumes access to a causally sufficient bivariate causal discovery algorithm. 
Following on from the discussion of causal sufficiency at the end of Section~\ref{Section:motivating_example}, we justify this assumption based on the following: $(1)$ future causal discovery algorithms may increase the domain in which causal sufficiency holds; $(2)$ certain robustness tests for determining the presence of unobserved confounding are possible  \cite{kalainathan2018sam,goudet2017learning}; and $(3)$ expert domain knowledge---such as that provided by a medical professional---and intervention-based studies can also provide causally sufficient information.
We also assume that all datasets come from the same underlying data generating process, as in IOD, see Section~\ref{Section: related work}.

We require a method for checking that a MAG output by IOD encodes all causal information learned from each dataset, not just the conditional independence information. We hence require criteria for easily distinguishing each of the structures in Fig.~\ref{figure:causal sufficiency}.
This is so that we can \emph{graphically} check for consistency of causal information between a candidate MAG and the MAGs of the individual datasets.

Motivated by the manner in which IOD stores graphical information about conditional independence using two data structures, \textbf{Sepset} \& \textbf{IP}, 
our criteria for distinguishing all structures from Fig.~\ref{figure:causal sufficiency} involves exploiting m-separations---checked by the absence of an active path between nodes---in mutilated versions of a MAG. Our criteria are listed in Table \ref{tab:criteria}.


To see that each criterion uniquely picks out one structure from Fig.~\ref{figure:causal sufficiency}, consider criterion 1, which picks out Fig~\ref{figure:causal sufficiency}(a). Here, there is a directed arrow from $X$ to $Y$. If the incoming arrow to $Y$, or outgoing arrow from $X$, is removed, then there is no longer an arrow from $X$ to $Y$, hence $X,Y$ are m-separated due to the absence of an active path between them. Conversely, if the outgoing arrow from $Y$, or the incoming arrow to $X$, is removed, then there is still an arrow from $X$ to $Y$, hence they are not m-separated---there is an active path between them. The conjunction of these facts uniquely picks out Fig.~\ref{figure:causal sufficiency}(a) and specifies criterion 1. Similar argument apply to all remaining structures from Fig.~\ref{figure:causal sufficiency}.

To illustrate the utility of these criteria, recall our motivating example from Section~\ref{Section:motivating_example}. A candidate MAG for this example must have $(Y, Z)$ m-separated---that is, no active path can exist between them---when both their incoming edges are removed. This fact is represented by criterion 3 from Table~1. Furthermore, $(Y, X)$ must be m-separated when the outgoing edges from $Y$ are removed, which is a condition of criterion 1 from Table~1. One of the possible solutions from section~\ref{Section:motivating_example}, Fig.~\ref{figure:motivating example}(b), violates criterion 3 for $(Y,Z)$, as $Y$ and $Z$ are not m-separated when their incoming edges are removed.
The other possible solution, Fig.~\ref{figure:motivating example}(c), violates criterion 1 for $(Y,X)$, as $Y$ and $X$ are not m-separated when $Y$'s outgoing edge is removed. In short, the criteria are simple graphical conditions---checked by the presence or absence of an active path in mutilated versions of the graph---applied to MAGs output by IOD to check consistency with information learned from causal discovery.

Our new algorithm proceeds as follows. First, part 1 of the IOD algorithm, described in Section~\ref{subSection: iod algorithm}, is applied to output (partially unoriented) graphs of each individual dataset, and a global graph. Next, bivariate causal discovery is applied to each dataset, the edges oriented accordingly in all the above graphs.
The causal structure found between each pair of dependent variables is then stored in three new data  structures---\textbf{Directed}$((X,Y))$ for Fig.~\ref{figure:causal sufficiency} (a) \& (b), \textbf{Common}$((X,Y))$ for Fig.~\ref{figure:causal sufficiency} (c), and \textbf{DirectedCommon}$((X,Y))$ for Fig.~\ref{figure:causal sufficiency} (d) \& (e). 
The order of the variables conveys the direction of the arrow in \textbf{Directed} and \textbf{DirectedCommon}.
Part 2 of the IOD algorithm is then applied to obtain candidate solutions for the joint causal structure which are consistent with all conditional independences learned in part 1 of IOD. Finally, these candidate solutions are filtered for bi-variate causal consistency by checking that each pair of variables in the above three data structures has the required causal structure in the candidate MAG. This is achieved by checking if the criterion relevant to the data structure holds in the MAG. If the requisite criteria are satisfied for all the pairs of variables stored in the data structures, then the candidate graph is accepted as a consistent solution, otherwise it is discarded. The steps are outlined in Algorithm~\ref{algorithm: new causal rules}.

\begin{algorithm}[h]
\caption{}\label{algorithm: new causal rules}
\textbf{Input:} Overlapping datasets $\{D_1,\dots, D_n\}$, IOD algorithm, bivariate causal discovery algorithm $C$. 
\\
\textbf{Output:} Consistent joint causal structures.
\begin{algorithmic}[1]
\State Apply part 1 of the IOD algorithm to return graphs $\mathcal{G}, \mathcal{G}_1, \dots \mathcal{G}_n$.
\State For each $i=1,\dots,n$, apply $C$ to variables in $D_i$ with unoriented edges and orient them accordingly in $\mathcal{G}, \mathcal{G}_1, \dots \mathcal{G}_n$. \\
Store causal relations by placing pairs of variables in the appropriate data structure: \textbf{Directed}, \textbf{Common}, and \textbf{DirectedCommon}.
\\
Apply part 2 of IOD algorithm to $\mathcal{G}, \mathcal{G}_1, \dots \mathcal{G}_n$. 
\\
Iterate through every MAG in the PAGs output by above and check relevant criteria hold for variables stored in step $3$ -  by checking active paths in the relevant mutilated MAGs. 
\\
\textbf{If} every MAG in a given PAG output from step $5$ is consistent \textbf{return} PAG \\
\textbf{else} \textbf{return} consistent individual MAGs
\end{algorithmic}
\end{algorithm} 

Algorithm~\ref{algorithm: new causal rules} is \emph{sound} in that each returned MAG has the same marginal causal structure between variables as that learned from the datasets, and \emph{complete} in that if a MAG exists with the same marginal causal structure between all variables as that learned from the dataset, it is returned by Algorithm~\ref{algorithm: new causal rules}. Proofs are in the Appendix.

\begin{theorem}[Soundness]
Let $V_i,V_j\in D_k$ be variables in the dataset $D_k$. If the marginal causal structure between $V_i,V_j$ learned from $ D_k$ by the BCD algorithm is Fig.~\ref{figure:causal sufficiency}(z), for $z\in\{a,b,c,d,e\}$, then it is also the marginal structure between $V_i,V_j$ in every MAG output by Algorithm~\ref{algorithm: new causal rules}, for all $i,j,k$. 
\end{theorem}

\begin{theorem}[Completeness]
Let $\mathcal{H}$ be a MAG over variables $\mathcal{V}$. If $V_i,V_j\in \mathcal{V}_k$, and the marginalised causal structure between $V_i,V_j$ in $\mathcal{H}$ coincides with that learned from dataset $\mathcal{D}_k$ by the BCD algorithm, then $\mathcal{H}$ is one of the MAGs output by Algorithm~\ref{algorithm: new causal rules}. 
\end{theorem}


\section{Experiments} \label{Section: experiments}
We now compare Algorithm~\ref{algorithm: new causal rules}, which we refer to as Causal IOD, to normal IOD. Additionally, to determine whether criteria 1-5 offer improvement over a straightforward application of bivariate causal discovery directly after part 1 of IOD, we also compare performance to a modified version of Algorithm~\ref{algorithm: new causal rules}, in which steps 3 and 5-7 are omitted. We refer to this as IOD with bivariate causal discovery (IOD$+$BCD). As remarked in Section~\ref{Section: introduction}, due to the reliance on conditional independence, IOD struggles when both the number of overlapping variables and the number of variables in each individual dataset are small. 
We verify this with the experiments and show that our method alleviates this issue.
We also test and compare our method in the high overlap, high number of variables regime.
The experiments are:
\begin{itemize}
    \item Synthetic 1 with low number of total variables and a small overlap set.
    \item Synthetic 2 with large number of total variables but low number of overlap variables.
    \item Protein dataset with high number of variables and a high overlap set.
    \item Cancer dataset with low number of variables with small overlap.
    \item Comparison of algorithms as a function of overlap on random graphs.
    \item Analysis of effect of sample size on Causal IOD.
\end{itemize}

For comparison, we measure the number of consistent MAGs output by each algorithm.
This is motivated by the fact that given certain CI and causal information, we desire the smallest number number of consistent MAGs to reduce the number of possible solutions that need to be verified experimentally or against domain experts to obtain the true data generating causal graph.
Note that for all the algorithms to output the true graph, the CI and BCD tests must find the correct information.
Thus for experiments with access to the true graph, we also measure two extra metrics: precision P $\equiv$ (No. of edges in MAGs with orientation as in ground truth / No. of edges in MAGs) and recall R $\equiv$ (No. of edges in MAGs with orientations as in ground truth / (No. of edges in ground truth $\times$ No. of MAGs).
Both are 1 when the only output is the true data generating MAG.
 

For the synthetic experiments, the partition of variables were chosen so that the local graphs satisfy causal sufficiency with respect to KCDC \cite{mitrovic2018causal}. This was done in order to compare performance when criteria 1-5 can be safely applied with KCDC as causal discovery. The functional relationships for the synthetic experiments are outlined in Appendix. Sample sizes of 3000 were used. HSIC \cite{gretton2008kernel} was used for marginal independence tests and KCIT
\cite{zhang2011kernel} for conditional independence.
The above algorithms used the RBF kernel with median data distance as the scale. Results are presented in Table~\ref{tab:results}.


\begin{figure}[t]   
\centering
\begin{subfigure}
    \centering
    (a)\includegraphics[scale=.1]{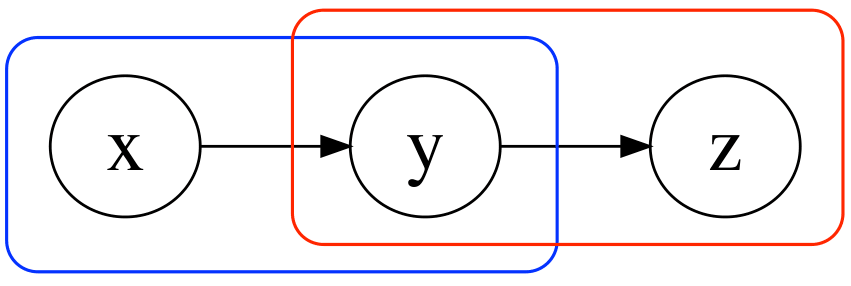} \\
    \label{fig:synth_master_2}
\end{subfigure}
\begin{subfigure}
    \centering
    (b)\includegraphics[scale=.1]{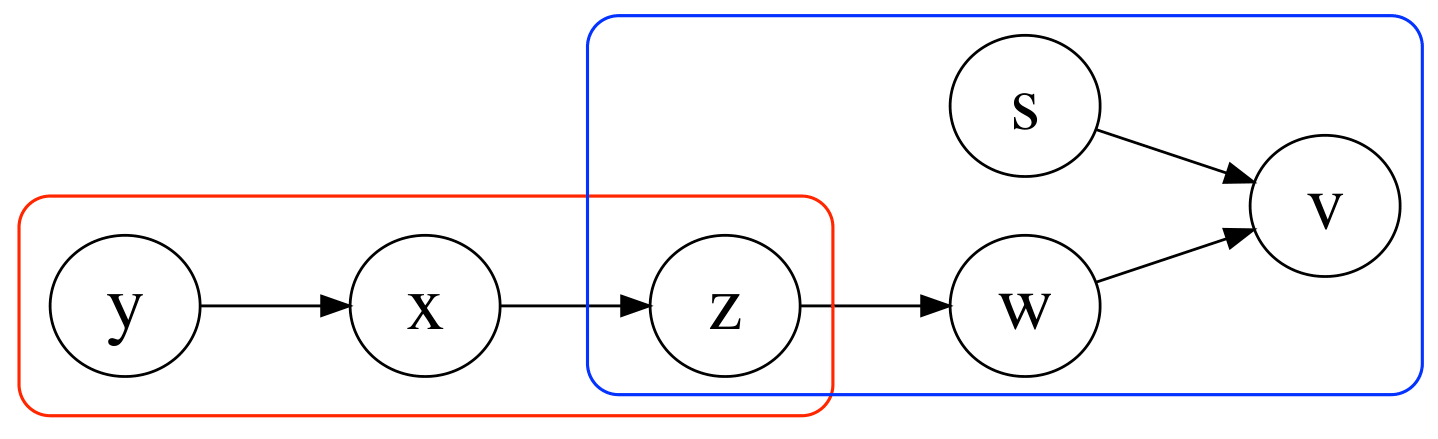}
    \label{fig:synth_master_1}
\end{subfigure}
\caption{The data generating graphs for experiments (a) Synthetic 1 \& (b) Synthetic 2. Boxes represent the node splits for the individual datasets.}\label{figure:synthetic}
\end{figure}



 \textbf{Synthetic 1.} 
Data is sampled from a graphical structure, in the low variable regime, with three variables as shown in Fig. \ref{figure:synthetic}(a). This was split into two datasets with variable $Y$ as the overlap.
The number of consistent MAGs, P and R scores are shown in Table~\ref{tab:results}.
The IOD algorithm returns the most number of MAGs and has the lowest P and R scores. 
This is due to its sole reliance on CI information to find consistent MAGs.
Any graph that encodes correlations between $X,Y$ and $Y,Z$ is consistent with CI information in the local graphs.
Since any causal direction or a common cause between the variables encodes correlations, the IOD algorithm returns every possible MAG between the three variables.
This is undesirable as it does not ascertain any information about the causal relationship between the variables that were not measured jointly.
IOD$+$BCD reduces the number of MAGs and results in an increase in the P and R scores.
This is due to the fact that, as a consequence of the arrows $X \rightarrow Y$ and $Y \rightarrow Z$ being oriented, certain orientations can be ruled out.
For example, it is no longer possible to have a MAG with an arrow $Z \rightarrow X$ as this results in a cyclic graph.
Causal IOD performs the best as it rules out
any candidate MAG that would violate the m-separations of criterion 1 for $X \rightarrow Y$ and $Y \rightarrow Z$.
For example, an arrow $X \rightarrow Z$ would now be invalid as it violates criterion 1 for $Y \rightarrow Z$.

\textbf{Synthetic 2.} Data is sampled from a graphical structure in Fig.~\ref{figure:synthetic}(b) of six variables that was split into two datasets with only one variable $Z$ as the overlap.  
In this low overlap regime, the computation of IOD proved intractable.
This was due to the fact that the small overlap in this case created too many edges to remove and immoralities to orient (see Section 2 in Appendix, here first iteration of \textbf{immoralities to orient} is over the powerset of a set of length 29 for IOD and only 8 for Causal IOD). As the bivariate causal discovery orients all edges in the local graphs, it substantially reduces the number of immoralities to consider. 
This is because immoralities can now only be oriented with edges that have an arrow.
This reduces the size of the largest set of \textbf{immoralities to orient} $M$, saving considerable compute time. 
The results in Table~\ref{tab:results} show that the causal IOD returns the correct graph while IOD$+$BCD results in multiple answers. The reason causal IOD performs well is due to the fact that all variables have pairwise directed causes amongst them, as seen in Fig. \ref{figure:synthetic}(b).
This rules out any graph with a backdoor path between these variables, even if the graph respects the CI information of the data.  

\textbf{Overlap experiment.} The three algorithms are now compared on randomly generated graphs of six nodes created using the process described in \cite{melancon2000random}.
With this experiment, we wish to attribute any difference in performance to both the incorporation of bivariate causal information and to criteria 1-5, rather than any difference in implementation of bivariate causal discovery algorithm or conditional independence test. 
Thus we give all three algorithms access to the true conditional independence information and true bivariate causal discovery outcomes in an oracle fashion.
This places all algorithms on equal footing, i.e. any difference in performance is due to the algorithm itself and not to any imperfect implementation of CI or causal discovery test.
The number of overlap variables are varied and the number of resulting consistent MAGs counted.
The choice of overlap variables was fixed before the graph generation to ensure that the overlap was random.
The global graph was generated and then marginalised into two separate local graphs.
This is repeated with 20 random graphs, and results averaged. The results can be seen in Table.~\ref{tab:results}.
Causal information vastly decreases the number of consistent answers as the number of overlapping variables grows.

\begin{table*}[t]
\scriptsize
    \caption{Results} \label{tab:results}
    \centering
    \begin{tabular}{cccccc}
    \toprule
    \multicolumn{1}{c}{Experiment} & \multicolumn{1}{c}{Overlap/Total variables} & \multicolumn{1}{c}{Algorithms}  & \multicolumn{1}{c}{Number of consistent MAGs} & \multicolumn{1}{c}{P} & \multicolumn{1}{c}{R} \\
    \midrule
    \multirow{ 3}{*}{Synthetic 1} & \multirow{ 3}{*}{1/3} & IOD  & 63 & 0.34 & 0.53 \\
    \multicolumn{1}{c}{} & & IOD + BCD  & 8 & 0.58 & 0.69 \\
    \multicolumn{1}{c}{} & & Causal IOD  & \textbf{1} & \textbf{1.00} & \textbf{1.00} \\
    \midrule
    \multirow{ 3}{*}{Synthetic 2} & \multirow{ 3}{*}{1/6} & IOD  & -  & - & -\\
    \multicolumn{1}{c}{} & & IOD + BCD  & 169 & 0.40 & 0.54 \\
    \multicolumn{1}{c}{} & & Causal IOD & \textbf{1} & \textbf{1.00} & \textbf{1.00}  \\
    \midrule    
    \multirow{ 9}{*}{Overlap} & \multirow{ 3}{*}{2/6} & IOD & 650.5 & 0.73 & 0.84 \\
    \multicolumn{1}{c}{} & & IOD + BCD  & 27.0 & 0.91 & 0.89  \\
    \multicolumn{1}{c}{} & & Causal IOD  & \textbf{15.9} &  \textbf{0.93} & \textbf{0.90} \\
    \cdashline{2-6}
     & \multirow{ 3}{*}{3/6} & IOD & 402.0 & 0.76 & 0.87      \\
    \multicolumn{1}{c}{} & & IOD + BCD  & 5.8 &  \textbf{0.96} & 0.94\\
    \multicolumn{1}{c}{} & & Causal IOD  & \textbf{5.3} & \textbf{0.96} & \textbf{0.95} \\
    \cdashline{2-6}
     & \multirow{ 3}{*}{4/6} & IOD & 230.5 & 0.77 & 0.91 \\
    \multicolumn{1}{c}{} & & IOD + BCD  & 1.9 &  \textbf{0.98} & 0.98 \\
    \multicolumn{1}{c}{} & & Causal IOD  & \textbf{1.8} &  \textbf{0.98} & \textbf{0.99} \\
    \midrule
    \multirow{ 6}{*}{Sample size} & \multirow{ 6}{*}{9/11} & IOD & 275 & 0.68 & 0.82 \\
    \multicolumn{1}{c}{} & & Causal IOD (200 samples)  & 2 & 0.78 & 0.88 \\
    \multicolumn{1}{c}{} & & Causal IOD (400 samples)  & 1 & 1.00 & 1.00 \\
    \multicolumn{1}{c}{} & & Causal IOD (600 samples)  & 1 & 1.00 & 1.00 \\    
    \multicolumn{1}{c}{} & & Causal IOD (800 samples)  & 1 & 1.00 & 1.00 \\
    \multicolumn{1}{c}{} & & Causal IOD (1000 samples)  & 1 & 1.00 & 1.00 \\
    \midrule
    \multirow{ 3}{*}{Protein} & \multirow{ 3}{*}{9/11} & IOD  & 61740 & 0.40 & 0.30\\
    \multicolumn{1}{c}{} & & IOD + BCD & 13  & \textbf{0.50}
    & \textbf{0.37} \\
    \multicolumn{1}{c}{} & & Causal IOD & \textbf{3} & \textbf{0.50} & \textbf{0.37} \\
    \midrule
    \multirow{ 3}{*}{Cancer} & \multirow{ 3}{*}{1/3} & IOD  & 42 \\
    \multicolumn{1}{c}{} & & IOD + BCD  & \textbf{6} \\
    \multicolumn{1}{c}{} & & Causal IOD  & \textbf{6} \\
    \hline
    \end{tabular}
\end{table*}

\textbf{Sample size.} We took data generated from a random graph of 7 variables and split it into two datasets with an overlap of 5 variables.
We then extracted the true conditional independence information from the two graphs (by checking m-separations).
This was so that both the algorithms had access to any CI information it needed---it also meant that the IOD algorithm had all the information it needed.
This provides a comparison of how sensitive Causal IOD is to sample size when compared against perfect implementation of IOD.
Data with varying sample sizes was sampled from the graph and the Causal IOD algorithm run.
The results show that with a small sample size of 200, KCDC  gives imperfect information resulting in a lower P and R score. However this is still higher than the IOD algorithm with access to all the information it needs. On all other sample sizes causal IOD achieves perfect performance. 

\textbf{Protein.} Next, the algorithms are compared on the Sachs et al. protein dataset.
The ground truth graph was taken from \cite{sachs2005causal}.
Here, Sachs et al. perturbed different proteins, observing the responses of other proteins.
Note the low P and R scores for this example is due to the CI tests declaring some variables as independent, that are not independent in literature.
This was noticed as well in \cite{goudet2017learning}, where they gave their algorithm access to the skeleton of the ground truth graph (hence the correct edges).
The incorrect edges found by the CI test then also affects the outcome of the BCD algorithm.
The similar P and R scores of the IOD+BCD and Causal IOD can also be attributed to this.


\textbf{Breast cancer.} We test on Breast Cancer data,\footnote{https://archive.ics.uci.edu/ml/datasets/Breast+Cancer
+Wisconsin+(Diagnostic)}
containing 10 features that describe the cell nucleus present in an image of a breast mass.
The images are associated with breast cancer diagnosis (Malignant or Benign).
Three variables---Diagnosis, Perimeter, \& Texture---were chosen and partitioned into two datasets with Diagnosis as the overlapping variable.
One might argue that causal sufficiency can be assumed as we do not expect Diagnosis and a physical feature to be confounded by the other physical feature. However, causal IOD \& IOD$+$BCD yield the same answers here, hence criteria 1-5 provide no advantage in this case. Note that no ground truth graph exists here.

\section{Conclusion}\label{Section: conclusion}
Here, we devised a new sound and complete algorithm for discovering causal structure from overlapping datasets using bivariate causal discovery. Our approach resulted in fewer MAGs than the current state-of-the-art algorithm, IOD \cite{tillman2011learning}---even when both the number of overlapping variables and the number of variables in each dataset were small. This smaller set of MAG makes it easier for domain experts to find the true graph, or to suggest further experiments to find it. 


Other approaches to integrating overlapping datasets use SAT solvers to check
is a candidate solution is consistent with conditional independences \cite{tsamardinos2012towards,triantafillou2010learning}. Future work will combine with bivariate causal discovery.

The main observation of our work is that local causal structure limits global structure. 
This is reminiscent of monogamy relations studied in quantum cryptography \cite{lee2018towards}, where local causal information limits how well eavesdroppers can intercept communications. Inspired by this, extensions to the growing field of quantum causal models \cite{allen2017quantum,lee2018towards,lee2018certification,chaves2015information} will be investigated. 
  
\bibliography{references} 
\bibliographystyle{ieeetr} 


\appendix
\section{Background information} \label{Section: background information}

We define a mixed graph $\mathcal{G} = \big< \mathcal{V}, \mathcal{E}\big>$, with vertices $\mathcal{V}$ and edges $\mathcal{E}$, as a graph containing three types of edges:
directed $\rightarrow$, undirected $-$, and bidirected $\leftrightarrow$.
Two nodes that share an edge are \textit{adjacent}.
A \textit{path} is defined as a sequence of nodes $ \langle V_1 ... V_i ... V_n \rangle$ such that $V_i$ and $V_{i+1}$ are adjacent for all $i$ and no node is repeated in the sequence.
A path is \textit{directed} if it follows the direction of the arrows. $X$ is an ancestor of $Y$ if there exists a directed path from $X$ to $Y$. 
$Y$ is then referred to as a descendent of $X$. An ancestor is a \emph{parent} if there are no intermediate nodes in the path, the direct descendent is then a \emph{child}. 
In graph $\mathcal{G}$, the ancestors of $X$ are denoted $\mathbf{Anc}_{\mathcal{G}}^{X}$, and descendents $\mathbf{Des}_{\mathcal{G}}^{X}$.  

A path is a \textit{collider} $\langle V_{i-1}, V_i, V_{i+1} \rangle$ if $V_{i-1}$ and $V_{i+1}$ both have a directed edge pointed at $V_i$.
A collider is then an \textit{immorality} if $V_{i-1}$ and $V_{i+1}$ are not adjacent.
A path between $X, Y$ is \textit{active} with respect to a set of nodes $\mathbf{Z}$ in graph $\mathcal{G}$ with $\{X,Y\} \notin \mathbf{Z}$, if:
(1) $\langle V_{i-1}, V_i, V_{i+1} \rangle$ is a collider in the path then $\{\{V_i\} \cup \mathbf{Des}^{V_i}_{\mathcal{G}} \} \cap \mathbf{Z} \neq \emptyset$, and
(2) If $V_i \in \mathbf{Z}$ then $\langle V_{i-1}, V_i, V_{i+1} \rangle$ is a collider.
In a graph $\mathcal{G}$, two nodes are \textit{m-separated} given $\mathbf{Z}$ if there does not exist an active path between them with respect to $\mathbf{Z}$, denoted $\mathbf{msep}_{\mathcal{G}}(X, Y| \mathbf{Z})$.
Closely related to m-separation is the graph concept of \textit{inducing paths}. An inducing path between nodes $X$, $Y$ relative to $\mathbf{Z}$ in a graph $\mathcal{G}$ is a path $\langle X, ... V_i, ... Y \rangle$ such that: (1) If $V_i \notin \mathbf{Z}$ then $\langle V_{i-1}, V_i, V_{i+1}\rangle$ is collider, and
(2) If $\langle V_{i-1}, V_i, V_{i+1}\rangle$ is a collider then $V_i \in \mathbf{Anc}^{X}_{\mathcal{G}} \bigcup  \mathbf{Anc}^{Y}_{\mathcal{G}}$.
If there is an inducing path between nodes, they cannot be m-separated.

A \textit{maximal ancestral graph} (MAG) $\mathcal{G} = \langle \mathcal{V}, \mathcal{E} \rangle$ is a mixed graph that is:
(1) \textbf{ancestral}: The graph is acyclic and does not have arrows pointing into nodes with an undirected edge ($X - Y$), (2) \textbf{maximal}: For any distinct nodes $V_i, V_j \in \mathcal{V}$, if $V_i, V_j$ are not adjacent in $\mathcal{G}$, then $\mathcal{G}$ does not contain any inducing paths between them with respect to the empty set. A DAG is just a MAG with directed edges. In addition to the independences encoded by a DAG, MAGs allow for latent variables that may be confounders---using a bidirected edge, or selection variables---using an undirected edge.

In this work we assume \textit{faithfulness} holds. That is, a MAG encodes an m-separation $\mathbf{msep}_{\mathcal{G}}(X, Y|\mathbf{Z})$ if and only if there exists a probability distribution on ${\mathcal{G}}$, $\mathbb{P}_{\mathcal{G}}$, in which $X$ is independent of $Y$ given $\mathbf{Z}$, denoted as $X \CI Y | \mathbf{Z}$.
There will usually be more than one MAG that can encode the same conditional independence information of a distribution $\mathbb{P}$.
Such MAGs are said to be \textit{Markov equivalent} and belong to the same \textit{Markov equivalence class}.
Two MAGs, $\mathcal{G} = \langle \mathcal{V}, \mathcal{E} \rangle$, and $\mathcal{H} = \langle \mathcal{W}, \mathcal{F} \rangle$ are Markov equivalent if they contain the same adjacencies, immoralities, and discriminating paths (defined in \cite{Zhang:2007:CME:3020488.3020543}).
If $\mathcal{W} \subset \mathcal{V}$, $\mathcal{H}$ is said to be a \emph{marginal} of $\mathcal{G}$ if the following holds for every $X, Y \in \mathcal{W}$: (1) If nodes $X$ and $Y$ are adjacent in $\mathcal{H}$, then they must have an inducing path in $\mathcal{G}$ with respect to $\mathcal{V} \backslash \mathcal{W}$, (2) For every \textbf{Z} that m-separates $X$ and $Y$ in $\mathcal{H}$, $X$ and $Y$ are also  m-separated by \textbf{Z} in $\mathcal{G}$.



We follow the \emph{partial ancestral graph} (PAG) convention for the graphical representation of Markov equivalent MAGs \cite{Zhang:2007:CME:3020488.3020543}.
Here, an edge type is considered invariant if all the MAGs in the Markov equivalence class have the same edge.
An arrowhead and tail are only represented in the PAG if they are invariant in the entire class.
The rest of the edges are represented by a circle $\circ$ symbolising that there are at least two MAGs in the Markov equivalence class with different edge types between the same variables.

\section{IOD algorithm: Complexity of loops} 
The IOD algorithm requires a nested loop over the powerset of sets \textbf{edges to remove} and within it, iterations over the powerset of \textbf{immoralities to orient}.
For an \textbf{edges to remove} set of size $n$ and, within it, an \textbf{immoralities to orient} set of size $m$, this step yields a complexity
of $O(2^{n+m})$.
However, as the size of \textbf{immoralities to orient} changes in each iteration of \textbf{edges to remove}, the complexity is proportional to the size of the largest \textbf{immoralities to orient} set $M$. Thus the complexity is $O(2^{n+M})$.

\section{Proofs of Soundness and Completeness}

Algorithm 1 from our main paper is \emph{sound} in that each returned MAG has the same marginal structure between variables as that learned from the datasets, and \emph{complete} in that if a MAG exists with the same marginal structure between all variables, it is returned by Algorithm 1.

\begin{theorem}[Soundness]
Let $V_i,V_j\in D_k$ be variables in the overlapping dataset $D_k$. If the marginal causal structure between $V_i,V_j$ learned from $ D_k$ is Fig.~\ref{figure:causal sufficiency1}(z), for $z\in\{a,b,c,d,e\}$, then it is also the marginal structure between $V_i,V_j$ in every MAG output by Algorithm 1, for all $i,j,k$. 
\end{theorem}

\begin{figure}[t]
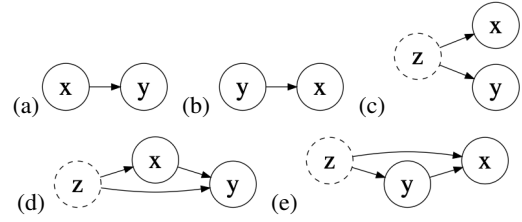
 
\centering
\begin{subfigure}
\centering
(a)\includegraphics[scale=.24]{directed1.png}
\end{subfigure}
\begin{subfigure}
\centering
(b) \includegraphics[scale=0.24]{directed2.png}
\end{subfigure}
\centering
\begin{subfigure}
(c) \includegraphics[scale=0.24]{common.png} \\
\end{subfigure}
\centering
\begin{subfigure}
(d) \includegraphics[scale=0.24]{commondirected1.png}
\end{subfigure}
\centering
\begin{subfigure}
(e) \includegraphics[scale=0.24]{commondirected2.png}
\end{subfigure}
\caption{All causal structures between 2 correlated variables, solid nodes observed \& dashed latent.}\label{figure:causal sufficiency1}
\end{figure}

\begin{proof}[\textbf{Proof of theorem 1}]
First, note that the IOD algorithm is sound \cite[Theorem 5.1]{tillman2011learning}, in that all MAGs output by IOD have the same conditional independence information as learned from $\mathcal{D}=\{D_1,\dots, D_n\}$. All that remains to check is whether a solution output by Algorithm 1 has the same marginal causal structure between each pair of jointly measured variables as learned from $\mathcal{D}$. The only situations that pose a potential problem are structures that are initial purely directed (Fig.~\ref{figure:causal sufficiency1} (a) or (b)) or common cause (Fig.~\ref{figure:causal sufficiency1} (c)), as they could become simultaneous direct-and-common (Fig.~\ref{figure:causal sufficiency1} (d) or (e)) when causal connections between non-jointly measured variables are added in part 2 of IOD. However, as criteria 1-5 can distinguish all of these cases, this problem does not arise.
\end{proof}

\begin{theorem}[Completeness]
Let $\mathcal{H}$ be a MAG over variables $\mathcal{V}$. If $V_i,V_j\in \mathcal{V}_k$ for $ 1 \leq k \leq n$ and the marginalised causal structure between $V_i,V_j$ in $\mathcal{H}$ coincides with that learned from $\mathcal{D}_k$, then $\mathcal{H}$ is one of the MAGs output by Algorithm 1. 
\end{theorem}

\begin{proof}[\textbf{Proof of theorem 2}]
Note that the IOD algorithm is complete \cite[Theorem 5.2]{tillman2011learning}, meaning all PAGs with the same conditional independence information as $\{D_1,\dots, D_n\}$ are output by IOD. Note also that applying bivariate causal discovery to a MAG does not change the Markov equivalence class it belongs to. The conjunction of these two facts implies Algorithm 1 is complete.
\end{proof}

\section{Functional relationships for Synthetic experiments}

For Synthetic 1, the functional relationships are

\begin{align*}
    x &= n_x \\
    y &= (3  \log x^2) \times n_y \\
    z &= (4y^2) \times n_z
\end{align*}

For Synthetic 2, the functional relationships were

\begin{align*}
    y &= n_y \\
    x &= (3  \log y^2) \times n_x \\
    z &= (4y^2) \times n_z \\
    w &= (z^{0.5}) \times n_w \\
    s &= n_s \\
    v &= (w^2 \times s^3) \times n_v
\end{align*}

where $n$ with a subscript denotes an exponential noise term with varying scale.

\section{Graph for sample size experiment}

Fig. \ref{figure:sample_size_graph} shows the graph used to generate the data used for the sample size experiment.
The two datasets contained the variables $\mathcal{V}_1 = \{y,x,t,z,u,v\}$ and $\mathcal{V}_2 = \{y,x,z,u,v,w\}$

\begin{figure} [!htb]
\centering
\includegraphics[scale=.29]{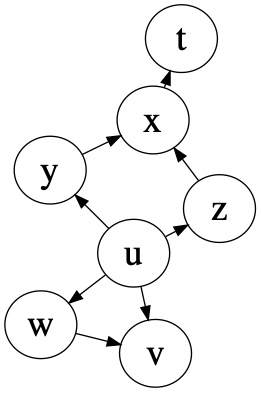} \\
\caption{Data generating graph for sample size experiment.}\label{figure:sample_size_graph}
\end{figure}

\section{Ground truth graph for the protein experiment}

The ground truth graph for the protein experiment \cite{sachs2005causal} is shown in Fig. \ref{figure:protein_gt}.
The data was split into two variable sets as follows: $\mathcal{V}_1 = $\{praf, pmek, plcg, PIP3, PIP3, pakts473, pjnk, PKC, P38, p44/42\} and $\mathcal{V}_2 =$\{praf, pmek, plcg, PIP3, PIP3, pakts473, pjnk, PKC, P38, PKA\}.

\begin{figure}[!htb]
\centering
\includegraphics[scale=.22]{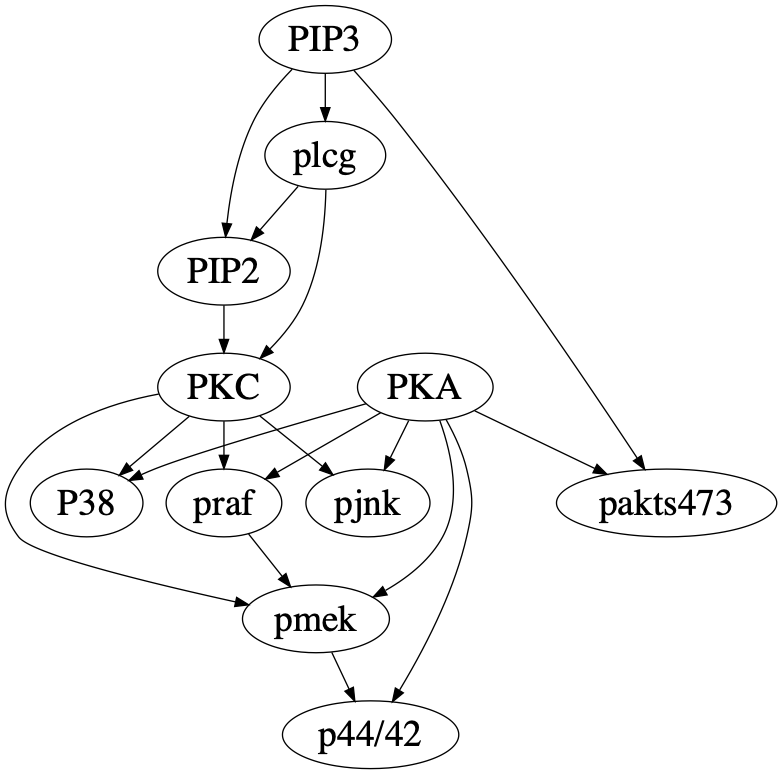} \\
\caption{True graph for protein experiment from \cite{sachs2005causal}}\label{figure:protein_gt}
\end{figure}


\end{document}